\documentclass{article}

\usepackage[preprint]{main}
\usepackage{tcolorbox}

\usepackage[utf8]{inputenc} 
\usepackage[T1]{fontenc}    
\usepackage{hyperref}       
\usepackage{url}            
\usepackage{booktabs}       
\usepackage{amsfonts}       
\usepackage{nicefrac}       
\usepackage{microtype}      
\usepackage{xcolor}         
\usepackage{amsmath}
\usepackage{amsthm}
\usepackage{algorithm}
\usepackage{algorithmic}
\newtheorem{lemma}{Lemma}
\newtheorem{theorem}{Theorem}

\theoremstyle{definition}
\newtheorem{assumption}[theorem]{Assumption}
\theoremstyle{plain}
\usepackage{booktabs}
\usepackage{multirow}
\usepackage{makecell}
\usepackage{amssymb}
\usepackage{listings}
\usepackage{caption}

\title{Interactive Critique-Revision Training for Reliable Structured LLM Generation}

\author{%
    Fei Xu Yu\\
    The George Washington University\\
    \texttt{fxyu@gwu.edu}\\
    \And
    Zuyuan Zhang\\
    The George Washington University\\
    \texttt{zuyuan.zhang@gwu.edu}\\
    \And
    Mahdi Imani\\
    Northeastern University\\
    \texttt{m.imani@northeastern.edu}\\
    \And
    Nathaniel D. Bastian\\
    United States Military Academy\\
    \texttt{nathaniel.bastian@westpoint.edu}\\
    \And
    Tian Lan\\
    The George Washington University\\
    \texttt{tlan@gwu.edu}
}

\begin{document}

\maketitle

\begin{abstract}
In structured decision-making workflows such as form filling, compliance checking, and maintenance reporting, LLM outputs must be locally correct, globally consistent, and auditable against task-specific rules. Existing refinement methods often rely on heuristic debate, self-play, or LLM-generated supervision, creating a second-order assurance problem. We propose \textbf{DPA-GRPO} (Dual Paired-Action Group-Relative Policy Optimization), a paired-action training method for a two-player generator--verifier game with structured verifier interventions. The generator proposes outputs and may revise them when challenged; the verifier either remains silent or raises a safety assurance case (SAC) containing a claim, argument, and evidence. These SAC/no-SAC and KEEP/REVISE decisions induce paired counterfactual action groups, which DPA-GRPO uses for role-specific KL-regularized GRPO updates. We analyze the unregularized game and show that positive probability on strictly lower-reward intervention or revision actions creates a profitable unilateral deviation. Under standard stochastic-approximation assumptions, DPA-GRPO tracks the corresponding game ODE, whose isolated asymptotically stable limit points are stationary and candidate local equilibria under role-wise local optimality. Experiments on TaxCalcBench TY24 show that DPA-GRPO improves structured decision accuracy over zero-shot generation and generator-only RL baselines across Qwen3-4B and Qwen3-8B. Training increases correct silent acceptance, reduces missed errors, and improves calibrated revision behavior, indicating gains for both generator and verifier.

\end{abstract}
\section{Introduction}

Large language models (LLMs) have achieved strong performance on open-ended generation, reasoning, and tool-use tasks~\citep{brown2020language,achiam2023gpt,wei2022chain,yao2022react,schick2023toolformer}, yet their use in complex, structured decision-making workflows remains brittle. 
In domains such as form filling, compliance checking, maintenance reporting, and program-like reasoning, LLM outputs must be plausible, locally correct, globally consistent, and auditable against required rules and standards~\citep{ji2023survey,huang2025survey,cobbe2021training,lightman2023let}. Addressing this mismatch is crucial for high-stakes structured workflows.

Recent progress highlights useful ingredients for this problem and leaves open how to train verifier intervention and generator revision policies around structured assurance evidence. First, using LLMs to directly verify or supervise other LLMs creates a second-order verification problem in which the verifier itself requires assessment. A verifier rewarded for finding errors may produce persuasive objections with weak grounding \citep{ji2023survey,huang2025survey,bai2022constitutional,madaan2023self,shinn2023reflexion}. A generator that follows verifier feedback uncritically may over-correct accurate outputs. A generator that discounts useful feedback may preserve unsafe errors. Second, multi-agent debate improves reasoning by allowing model instances to challenge one another. Existing implementations often use debate as a prompting or test-time scaling mechanism rather than as a role-specific training paradigm \citep{du2024improving,khan2024debating}, without a systematic framework to co-train or co-adapt all agents in concert \citep{cobbe2021training,lightman2023let,lifshitz2025multi,jung2026coevolvingagentslearningfailures,yu2025optimizing}. Finally, recent work on self-play and competitive evaluation methods shows promise for learning from game outcomes and inter-model competition, and leave open safety-case-style assurance with respect to specific requirements and auditability~\citep{kuba2025languageselfplaydatafreetraining,alyahya2025zerosumeval,liu2025spiral, liao2024efficacylanguagemodelselfplay}. The challenges to get AI agents to agree are highlighted in~\citep{berdoz2026aiagentsagree}.

We propose \textbf{DPA-GRPO} (Dual Paired-Action Group-Relative Policy Optimization), a game-theoretic solution for safety-case-based assurance of structured LLM outputs. Game theory provides a mathematical framework for analyzing strategic interaction between learning agents, and a well-designed equilibrium pushes them toward desired collective behaviors. We train a two-player Markov game \citep{shapley1953stochastic} between a \emph{generator} and a \emph{verifier}. The generator proposes solutions (e.g., form filling or multi-step planning) and may revise them when challenged. The verifier observes the proposed solutions and chooses whether to remain silent or raise a safety assurance case (SAC) consisting of a claim, argument, and evidence. This idea is
closely related to a long tradition of safety certification. Safety-critical
engineering has long used \emph{safety cases}: structured
claim-argument-evidence records that justify why a system satisfies required
properties and is sufficiently safe for deployment
\citep{kelly2004goal,denney2015dynamic,dezfuli2015nasa,lincoln2012overview}. Such records are attractive for LLM validation because they make objections interpretable: the verifier states a claim, supports it with an argument, and grounds the argument in evidence. This structure provides a structured interface for verifier challenges and generator revisions. A more detailed discussion of related work is provided in Appendix~\ref{sec:relatedwork}.

The generator-verifier game turns structured assurance feedback into paired intervention and revision decisions, producing paired counterfactual training signals at each evaluated output. In particular, SAC/no-SAC decisions and keep/revise responses arising from the game produce novel positive/negative samples that do not exist in standard training regimes, such as SFT, DPO, or GRPO \citep{ouyang2022training,rafailov2023direct, shao2024deepseekmath}.
We analyze this game and categorize the counterfactual training signals into eight informative cases. Based on these signals, DPA-GRPO trains the generator and verifier in parallel. The resulting paired-action updates are related to recent interpretations of GRPO as a contrastive or preference-like objective
\citep{shao2024deepseekmath,wu2025takes}, but differ in that the preference pairs are induced by structured generator--verifier interactions rather than external preference labels. Both agents are trained with role-conditional policies and KL-regularized
objectives, allowing the system to improve through interaction while limiting
unstable policy drift.

Our theoretical analysis studies necessary conditions for local Nash equilibria in the two-player Markov game~\citep{nash1950equilibrium,shapley1953stochastic}. Desired behavior corresponds to a parameter configuration in which each agent is locally optimal for its own role-specific reward \citep{cobbe2021training,lightman2023let}. We show that under standard non-degeneracy assumptions, positive probability mass on lower-reward cases, such as missed violations or harmful unnecessary revisions, creates a profitable unilateral deviation in the unregularized game. With KL regularization, the same reward gap yields a soft best-response bound on lower-reward action probability. Under the stochastic-approximation assumptions, the interpolated DPA-GRPO iterates form an asymptotic pseudo-trajectory of the game ODE. Consequently, every isolated asymptotically stable limit point is a candidate local Nash equilibrium.

Our main contributions are as follows:
First, we formulate structured LLM validation as a two-player Markov game between a generator and a verifier, where verifier interventions are expressed as structured SACs containing claims, arguments, and evidence reasoning. Second, we propose DPA-GRPO, a dual paired-action training method that constructs counterfactual SAC/no-SAC and keep/revise comparisons to train both agents with role-conditional policies and KL-regularized updates. Third, we develop an interpretable case taxonomy and connect it to local Nash conditions and stochastic-approximation dynamics, providing a principled way to diagnose the training signals induced by generator--verifier interaction. Fourth, we evaluate DPA-GRPO on TaxCalcBench TY24 using structured decision accuracy and case-distribution diagnostics, showing improved performance over zero-shot generation and generator-only RL baselines.

\begin{figure}[t]
    \centering
    \includegraphics[width=\linewidth]{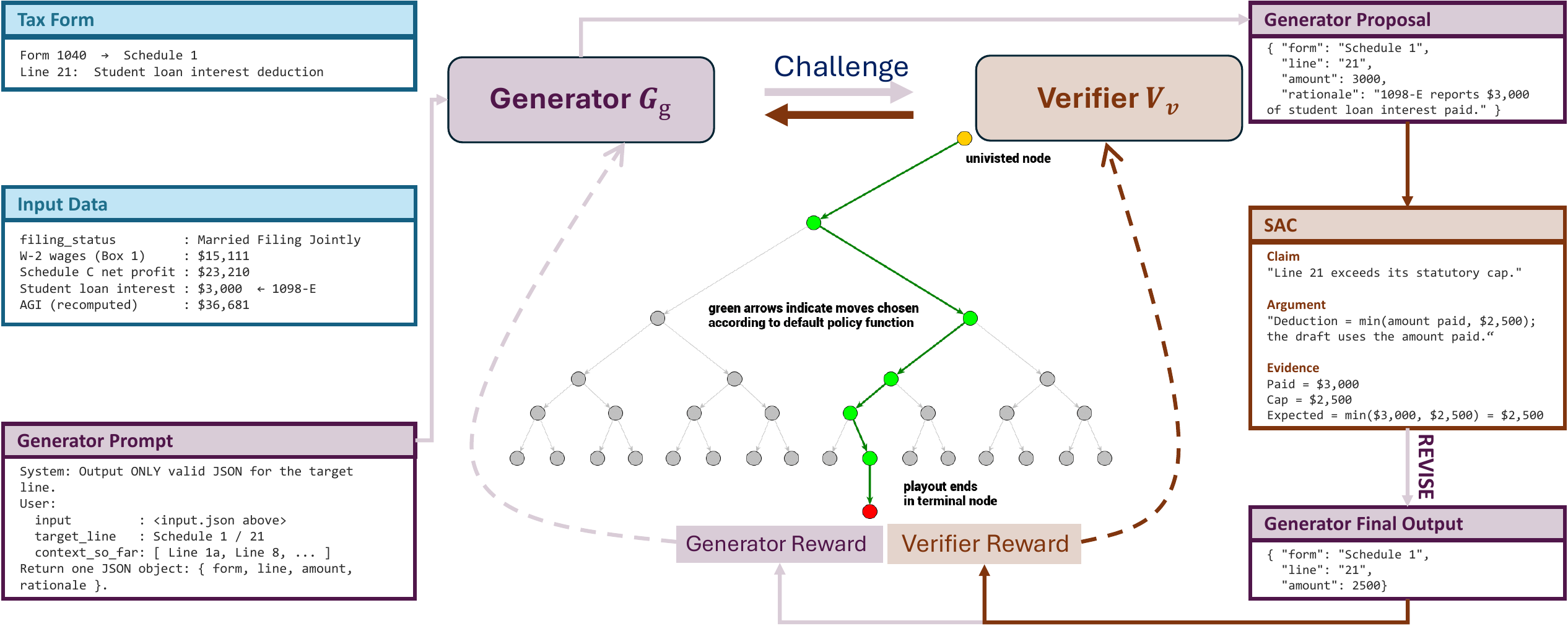}
    \caption{{\textbf{Illustration of the generator--verifier game in a tax-form completion task.}
    The left panel shows the application context for one evaluated decision unit, including the relevant form text, structured input data, and generator prompt. The generator first proposes an output value, after which the verifier may challenge the proposal by producing a structured safety assurance case (SAC) with a claim, argument, and evidence. The generator then decides whether to keep the original proposal or revise it. This interaction produces role-specific rewards for both agents, which DPA-GRPO uses to construct paired counterfactual training signals.}}
    
    \label{fig:case_histo}
\end{figure}

\section{A Generator--Verifier Markov Game}
\label{sec:game}

We model structured output generation as a finite-horizon two-player Markov game between a generator $G$ and a verifier $V$. A task instance $\tau\sim\mathcal{D}$ contains the structured input, task rules, output schema, and any dependencies required to complete the output. For each task $\tau$, the game unfolds over a finite horizon $t=1,\ldots, T_\tau$. At each time step $t$, the players complete one structured decision unit $u_t$. The state $m_t$ contains the task instance, the current decision unit, and the partial output generated so far:
$
m_t = (\tau,u_t,o_{<t}).
$

The two players have role-specific policies and action spaces. The generator first produces an initial proposal for the current decision unit,
$
x_t \sim \pi_G^x(\cdot \mid m_t).
$
The verifier observes the current state and proposal, $(m_t,x_t)$, and chooses an intervention action
$
y_t \sim \pi_V(\cdot \mid m_t,x_t), \text{where } y_t \in \{\mathrm{SAC},\mathrm{NS}\}.
$
Here, $\mathrm{SAC}$ denotes a structured safety assurance case and $\mathrm{NS}$ denotes no-SAC. An SAC may include a claim, supporting argument, evidence, and a suggested correction.

If the verifier chooses $\mathrm{NS}$, the proposal is submitted directly, so $o_t=x_t$. If the verifier chooses $\mathrm{SAC}$, the generator observes the verifier intervention and chooses how to respond. It may generate a candidate revision
$
z_t \sim \pi_G^z(\cdot \mid m_t,x_t,y_t)
$
and then select a revision action
$
a_t \sim \pi_G^a(\cdot \mid m_t,x_t,y_t,z_t), \text{where } a_t \in \{\mathrm{KEEP},\mathrm{REVISE}\}.
$
If $a_t=\mathrm{KEEP}$, the original proposal is submitted $(o_t=x_t)$; if $a_t=\mathrm{REVISE}$, the revised proposal is submitted $(o_t=z_t)$. The transition appends the submitted output component to the partial output and advances to the next decision unit:
$
m_{t+1} = T(m_t,x_t,y_t,a_t,o_t).
$
The episode terminates after all $T_\tau$ decision units have been completed.

This game is general-sum rather than zero-sum: both players are ultimately evaluated against the correctness of the structured output, but they receive role-specific rewards. The generator reward measures whether the submitted output component is correct and whether revision improves the initial proposal. The verifier reward measures whether the intervention decision is appropriate, rewarding useful SACs on incorrect proposals and penalizing missed errors or unnecessary interventions. Thus, the game defines not only a modeling abstraction for generator--verifier interaction, but also a source of training signals for both agents.

DPA-GRPO exploits this structure by constructing paired counterfactual action groups at each decision unit. For the verifier, the paired actions are $\{\mathrm{SAC},\mathrm{NS}\}$; for the generator, the paired actions are $\{\mathrm{KEEP},\mathrm{REVISE}\}$ when an intervention is raised. Because each decision unit can be checked against verifiable task-level labels or rules, these paired actions induce dense role-specific rewards. DPA-GRPO therefore turns each structured decision into training signals for both the verifier's intervention policy and the generator's revision policy.

\subsection{Game and SAC Taxonomy}
\label{sec:taxonomy}

We now define what is contained in a
SAC and how SAC outcomes are used to diagnose training behavior.

\paragraph{Safety Assurance Case (SAC).}
A structured safety assurance case (SAC) is the verifier's intervention message.
Its design follows safety-case methodology, where an assurance argument links a
claim to supporting evidence through an explicit reasoning step. NASA's safety-case-patterns report describes safety case patterns as reusable argument structures that connect types of
claims to available types of evidence and specify when the pattern applies
\citep{dezfuli2015nasa}. Related regulatory quality-system settings emphasize documented procedures, records, and controls as part of establishing and maintaining reliable processes \citep{ecfr21cfr820}. In our setting, an SAC functions as a structured verifier intervention that helps the generator decide whether the proposal should be revised.

Formally, when $y_t=\mathrm{SAC}$, the verifier emits
$
    \mathcal{SAC}_t=(c_t,g_t,e_t),
$
where $c_t$ is the \emph{claim} that the initial proposal is incorrect or requires review,
$g_t$ is the \emph{argument} explaining why the claim follows from the task context and
applicable rule, and $e_t$ is the \emph{evidence}, such as relevant input fields,
intermediate computations, or line dependencies. When $y_t=\mathrm{NS}$, the SAC field is empty. An example SAC consisting of claim, argument, and evidence for form filing is shown below.

\begin{tcolorbox}
\[
\begin{aligned}
c_t &: \text{``The initial proposal for line $t$ appears inconsistent with the input record.''}\\
g_t &: \text{``The line should include both components, but the initial proposal omits one of them.''}\\
e_t &: \text{``Component 1 = \$A, component 2 = \$B; expected value = \$A+\$B.''}
\end{aligned}
\]
\end{tcolorbox}
We separately evaluate SAC validity. Let $Q_{\mathrm{SAC}}(\mathcal{SAC}_t; m_t,x_t)\in[0,1]$ score whether the claim targets the correct line, the argument cites the applicable rule or dependency, and the evidence appears in the task context. 

\paragraph{Case taxonomy.}
We use case taxonomy to diagnose behavior and connect training dynamics to the
equilibrium conditions in Section~\ref{sec:main-guarantees}. Let
$S_x=\mathbf{1}\{x_t \text{ is strictly correct}\}$ and
$S_z=\mathbf{1}\{z_t \text{ is strictly correct}\}$. Each transition is classified by initial proposal
correctness $S_x$, verifier decision $y_t$, generator action $a_t$, and revision correctness
$S_z$:
\[
\begin{array}{c|cccc|l}
\toprule\toprule
\text{Case} & S_x & y_t & a_t & S_z & \text{Interpretation} \\
\hline
1 & 1 & \mathrm{NS}  & - & - & \text{correct initial proposal accepted} \\
2 & 1 & \mathrm{SAC} & \mathrm{REVISE} & - & \text{false-positive SAC, generator revises} \\
3 & 1 & \mathrm{SAC} & \mathrm{KEEP} & - & \text{false-positive SAC, generator keeps} \\
4 & 0 & \mathrm{NS}  & - & - & \text{missed initial proposal error} \\
5\mathrm{a} & 0 & \mathrm{SAC} & \mathrm{REVISE} & 1 & \text{successful repair} \\
5\mathrm{b} & 0 & \mathrm{SAC} & \mathrm{REVISE} & 0 & \text{failed repair accepted} \\
6\mathrm{a} & 0 & \mathrm{SAC} & \mathrm{KEEP} & 0 & \text{bad revision rejected} \\
6\mathrm{b} & 0 & \mathrm{SAC} & \mathrm{KEEP} & 1 & \text{good revision rejected}\\
\bottomrule
\end{array}
\]
Cases 1 and 5a represent the desired outcomes. Case 1 corresponds to a correct proposal accepted without
intervention, while Case 5a corresponds to a correct verifier intervention followed by a
successful repair. Case 4 indicates verifier under-intervention. Cases 2 and 3 measure verifier over-intervention and generator response quality. Cases 5b, 6a, and 6b diagnose revision-selection behavior: accepting a bad revision, rejecting a bad revision, and rejecting a good revision, respectively. Detailed generator--verifier interactions are provided in Appendix~\ref{app:case-interactions}.

\subsection{Paired Rewards and Regularized Objectives}
\label{sec:rewards}

During training, a reward oracle provides correctness labels. Let $S_x=\mathbf{1}\{x_t\text{ is strictly correct}\}$ and $S_z=\mathbf{1}\{z_t\text{ is strictly correct}\}$. We define the SAC-correctness label as $c_{\mathrm{sac}}=1-S_x$, so SAC is correct if and only if the initial proposal is incorrect. The oracle enters rewards and diagnostics only; agent policies receive the task context and interaction messages.

The verifier receives paired intervention rewards
\[
R_V(\mathrm{SAC}\mid m_t,x_t)=c_{\mathrm{sac}},\qquad
R_V(\mathrm{NS}\mid m_t,x_t)=1-c_{\mathrm{sac}}.
\]
On the SAC branch, the generator compares KEEP and REVISE using
\[
R_G(\mathrm{KEEP}\mid m_t,x_t,\mathrm{SAC})=S_x,\qquad
R_G(\mathrm{REVISE}\mid m_t,x_t,\mathrm{SAC})=S_z.
\]

Let $\widetilde J_f(\phi;\theta)=\mathbb{E}_{\tau\sim d^{\phi,\theta}}[R_f(\tau)]$ and $\widetilde J_v(\theta;\phi)=\mathbb{E}_{\tau\sim d^{\phi,\theta}}[R_v(\tau)]$ denote the unregularized expected role rewards. The KL-regularized objectives are
\[
J_v(\theta;\phi)=\widetilde J_v(\theta;\phi)-\beta_v\,\mathbb{E}_{I_v}\!\left[\mathrm{KL}\!\left(\pi_{v,\theta}(\cdot\mid I_v)\,\|\,\pi_{v,\mathrm{ref}}(\cdot\mid I_v)\right)\right],
\]
and
\[
\begin{aligned}
J_f(\phi;\theta)
=
\widetilde J_f(\phi;\theta)
&-\beta_{fx}\,\mathbb{E}_{I_f^x}\!\left[\mathrm{KL}\!\left(\pi^x_{f,\phi}(\cdot\mid I_f^x)\,\|\,\pi^x_{f,\mathrm{ref}}(\cdot\mid I_f^x)\right)\right] \\
&-\beta_{fz}\,\mathbb{E}_{I_f^z}\!\left[\mathrm{KL}\!\left(\pi^z_{f,\phi}(\cdot\mid I_f^z)\,\|\,\pi^z_{f,\mathrm{ref}}(\cdot\mid I_f^z)\right)\right] \\
&-\beta_{fa}\,\mathbb{E}_{I_f^a}\!\left[\mathrm{KL}\!\left(\pi^a_{f,\phi}(\cdot\mid I_f^a)\,\|\,\pi^a_{f,\mathrm{ref}}(\cdot\mid I_f^a)\right)\right].
\end{aligned}
\]
The KL terms stabilize training and are accounted for in the equilibrium analysis.

\subsection{Dual Paired-Action GRPO}
\label{sec:paired-grpo}

The central implementation choice is that the verifier decision and
generator-on-SAC decision are binary and counterfactually scorable. For the
verifier, the action group is $G_V(m_t,x_t)=\{\mathrm{SAC},\mathrm{NS}\}$.
For the generator on the SAC branch, the action group is
$G_G(m_t,x_t,y_t,z_t)=\{\mathrm{KEEP},\mathrm{REVISE}\}$.

For a context $I$ and binary group $G(I)=\{a_1,a_2\}$, we compute rewards
$R(a_1)$ and $R(a_2)$ and form the group-normalized advantage
\[
A(a_i)=
\frac{R(a_i)-\bar R}{\sqrt{\frac{1}{2}\sum_{j=1}^2(R(a_j)-\bar R)^2}+\epsilon_A},
\qquad
\bar R=\frac{1}{2}(R(a_1)+R(a_2)).
\]
With two actions this reduces to a contrastive pressure: the higher-reward
action receives positive advantage and the other receives negative advantage
of equal magnitude.

Each role's update is an on-policy policy-gradient step with a KL anchor
against a frozen reference policy. For role policy $\pi_\psi$ with reference
$\pi_{\psi,\mathrm{ref}}$, the paired-action loss is
\[
\mathcal{L}_{\mathrm{pair}}(\psi)
=
\mathbb{E}\!\left[
-A(a)\,\log\pi_\psi(a\mid I)
+\beta\!\left(\log\pi_\psi(a\mid I)-\log\pi_{\psi,\mathrm{ref}}(a\mid I)\right)
\right].
\]
The verifier update uses $\psi=\theta$ and $G=G_V$; the generator action
update uses $\psi=\phi$ and $G=G_G$. Losses are routed strictly by role: verifier rewards update only
the verifier policy and KEEP/REVISE rewards update only the generator
policy. When multiple revision candidates are sampled for the same decision context, the same paired-action structure applies independently on each sampled decision
unit and advantages are group-normalized per unit.

\subsection{Training Procedure}
\label{sec:training-procedure}

The full DPA-GRPO training loop is given in Algorithm~\ref{alg:training} in Appendix~\ref{sec:algorithm}. At each iteration, the current generator and verifier generate a fresh batch of decision-unit interactions. The batch records
$
(m_t,x_t,y_t,z_t,a_t,o_t,S_x,S_z,c_{\mathrm{sac}})
$
together with paired rewards and old log-probabilities. The batch is used for a REINFORCE-style paired-action policy-gradient update with a KL anchor and is then discarded for training.

The generator and verifier are trained as fixed roles throughout training. The generator policy produces initial proposals and candidate revisions. The verifier policy produces SAC/no-SAC decisions and SAC content. Losses are routed by decision type: verifier paired rewards update the verifier policy, and generator KEEP/REVISE rewards update the generator policy on SAC branches.

\section{Theoretical Guarantees}
\label{sec:main-guarantees}

We next connect the implemented rewards to local equilibrium behavior and convergence.
The main text states the core assumptions, lemmas, and theorem statements; full theorem-specific
technical assumptions and proofs are deferred to Appendix~\ref{app:theory}.

\paragraph{Local Nash equilibrium.}
A pair $(\phi^*,\theta^*)$ is a local Nash equilibrium if there exists a neighborhood
$\mathcal{N}$ such that
$
J_f(\phi^*;\theta^*)\ge J_f(\phi;\theta^*)\quad \forall \phi\in\mathcal{N},
\quad
J_v(\theta^*;\phi^*)\ge J_v(\theta;\phi^*)\quad \forall \theta\in\mathcal{N}.
$
At an interior local Nash equilibrium, the necessary first-order conditions are
$
\nabla_\phi J_f(\phi^*;\theta^*)=0,\quad
\nabla_\theta J_v(\theta^*;\phi^*)=0.
$
These are each player's own gradients; cross-gradients describe coupling but are not Nash
stationarity conditions.

\begin{assumption}[Local non-degeneracy]
\label{ass:local-nondegenerate}
For each binary role decision, whenever a policy assigns nonzero probability to both actions
on a positive-measure set of contexts, there exists a local parameter direction that transfers
an arbitrarily small amount of probability mass from one action to the other on that set.
\end{assumption}

\begin{lemma}[Verifier pointwise best response]
\label{lem:main-verifier-br}
Fix a context $(m_t,x_t)$. Under the verifier rewards
$
R_v(\mathrm{SAC}\mid  m_t,x_t)=c_{\mathrm{sac}},\quad
R_v(\mathrm{NS}\mid  m_t,x_t)=1-c_{\mathrm{sac}},
$
the unique pointwise reward-maximizing verifier action is
\[
y^*( m_t,x_t)=
\begin{cases}
\mathrm{SAC}, & c_{\mathrm{sac}}=1,\\
\mathrm{NS}, & c_{\mathrm{sac}}=0.
\end{cases}
\]
Consequently, if the verifier places nonzero probability on the opposite action on a positive-measure
set of contexts, then there exists a local direction $\Delta\theta$ such that
$
D_\theta \widetilde J_v(\theta;\phi)[\Delta\theta] > 0.
$
\end{lemma}

\begin{lemma}[Generator best response on the SAC branch]
\label{lem:main-generator-br}
On the SAC branch, define
$
\bar S_z(m_t,x_t)=
\mathbb{E}\!\left[
S_z\mid  m_t,x_t,\mathrm{SAC},a_t=\mathrm{REVISE}
\right].
$
KEEP has reward $S_x$, while REVISE has conditional expected reward $\bar S_z( m_t,x_t)$.
Therefore, if $S_x>\bar S_z( m_t,x_t)$, KEEP strictly dominates REVISE; if
$S_x<\bar S_z( m_t,x_t)$, REVISE strictly dominates KEEP. Consequently, whenever
the generator places nonzero probability on the strictly dominated action on a positive-measure
set of SAC contexts, there exists a local direction $\Delta\phi$ such that
$
D_\phi \widetilde J_f(\phi;\theta)[\Delta\phi] > 0.
$
\end{lemma}

\paragraph{Nash implication.}
Lemmas~\ref{lem:main-verifier-br} and~\ref{lem:main-generator-br} show that the implemented rewards
create pointwise best-response pressure: the verifier is pushed toward raising SAC exactly when the
proposal is wrong, and the generator is pushed toward KEEP or REVISE according to which action has higher
conditional correctness. Thus, positive mass on verifier mistakes or strictly suboptimal generator responses to SAC creates a unilateral reward-improving direction. The following theorem states the corresponding
necessary condition for local Nash equilibria of the KL-regularized game.

\begin{theorem}[KL-regularized paired-action best response]
\label{thm:main-lne}
Suppose $(\phi^*,\theta^*)$ is a local Nash equilibrium of the KL-regularized game. Under
Assumption~\ref{ass:local-nondegenerate} and the KL-dominance condition stated in
Appendix~\ref{app:lne-proof}, the verifier assigns zero probability to pointwise reward-suboptimal
SAC/no-SAC actions on every positive-measure set of contexts, and the generator assigns zero probability to
strictly suboptimal KEEP/REVISE actions on every positive-measure set of SAC contexts.
\end{theorem}

Theorem~\ref{thm:main-lne} shows how KL regularization suppresses lower-reward verifier and generator actions through reward-gap-dependent action odds. A Case~1-favoring result requires an additional intervention or revision cost: successful
SAC-and-revise behavior (Case~5a) can remain reward-consistent unless one adds a cost for relying
on revision.

\paragraph{Convergence of alternating GRPO.}
Let $F(\phi,\theta)=(\nabla_\phi J_f(\phi;\theta),\nabla_\theta J_v(\theta;\phi))$ denote
the game vector field. The alternating GRPO updates can be viewed as a stochastic approximation
to the ODE
$
\dot\phi=\nabla_\phi J_f(\phi;\theta),\quad
\dot\theta=\nabla_\theta J_v(\theta;\phi).
$
Under standard stochastic-approximation conditions, stated in Appendix~\ref{app:convergence-proof},
the interpolated iterates track this ODE.

\begin{theorem}[Stochastic-approximation tracking]
\label{thm:main-convergence}
Under the stochastic-approximation assumptions in Appendix~\ref{app:convergence-proof}, the
interpolated alternating-GRPO iterates form an asymptotic pseudo-trajectory of the game ODE.
Consequently, every isolated asymptotically stable limit point satisfies
$
\nabla_\phi J_f(\phi;\theta)=0,\quad
\nabla_\theta J_v(\theta;\phi)=0,
$
and is therefore stationary for the game vector field. A local Nash interpretation additionally requires role-wise local optimality conditions.
\end{theorem}
\providecommand{\PH}{--}
\providecommand{\bestnum}[1]{\textbf{#1}}     
\providecommand{\oursrow}[1]{\textbf{#1}}      

\section{Experiments}
\label{sec:experiments}

\begin{table}[ht]
\centering
\small
\setlength{\tabcolsep}{5pt}
\renewcommand{\arraystretch}{1.08}
\caption{
Headline comparison on TaxCalcBench TY24. We report test-set case-taxonomy diagnostics, training-time accuracy on three common training samples, and test accuracy on the fixed-seed 80/20 split. The case-taxonomy columns follow Section~\ref{sec:taxonomy}: Case~1 + Case~3 measures correct final outputs without harmful revision, Case~4 measures missed errors, and Case~6a measures correct rejection of bad revisions. For generator-only baselines, Case~1 + Case~3 reduces to correct generator outputs, Case~4 reduces to incorrect generator outputs, and Case~6a is not applicable. DPA-GRPO achieves the highest test accuracy at both backbone sizes.
}
\label{tab:main-results}

\resizebox{\linewidth}{!}{%
\begin{tabular}{l c c c c c c c c}
\toprule\toprule
\textbf{Method} & \textbf{Backbone} &
\multicolumn{3}{c}{\textbf{Case taxonomy}} &
\multicolumn{3}{c}{\textbf{Train acc.}} &
\makecell{\textbf{Test}\\\textbf{acc.}$\uparrow$} \\
\cmidrule(lr){3-5}
\cmidrule(lr){6-8}
& &
\makecell{\textbf{Case 1 + Case 3}$\uparrow$} &
\makecell{\textbf{Case 4}$\downarrow$} &
\makecell{\textbf{Case 6a}$\uparrow$} &
\makecell{\textbf{Train-A}$\uparrow$} &
\makecell{\textbf{Train-B}$\uparrow$} &
\makecell{\textbf{Train-C}$\uparrow$} &
\\
\midrule
\multicolumn{9}{l}{\emph{Untrained baselines}} \\
\midrule
Zero-shot & Qwen3-4B & 0.41 + 0 & 0.59 & \PH & 0.53 & 0.32 & 0.26 & 0.41 \\
Zero-shot & Qwen3-8B & 0.53 + 0 & 0.47 & \PH & 0.84 & 0.42 & 0.42 & 0.53 \\
\midrule
\multicolumn{9}{l}{\emph{Trained single-agent generator-only baselines}} \\
\midrule
GRPO & Qwen3-4B & 0.52 + 0 & 0.48 & \PH & 0.74 & 0.47 & 0.32 & 0.52 \\
DAPO & Qwen3-4B & 0.47 + 0 & 0.53 & \PH & 0.74 & 0.37 & 0.21 & 0.47 \\
GSPO & Qwen3-4B & 0.46 + 0& 0.54 & \PH & 0.74 & 0.37 & 0.21 & 0.46 \\
GRPO & Qwen3-8B & 0.54 + 0 & 0.46 & \PH & 0.84 & 0.63 & 0.32 & 0.54 \\
DAPO & Qwen3-8B & 0.55 + 0 & 0.45 & \PH & 0.84 & 0.63 & 0.53 & 0.55 \\
GSPO & Qwen3-8B & 0.52 + 0 & 0.48 & \PH & 0.84 & 0.68 & 0.58 & 0.52 \\
\midrule
\multicolumn{9}{l}{\emph{Ours}} \\
\midrule
\oursrow{DPA-GRPO} & Qwen3-4B & 0.52 + 0.01 & 0.40 & 0.06 & 0.74 & 0.58 & 0.37 & 0.53 \\
\oursrow{DPA-GRPO} & Qwen3-8B & \bestnum{0.43 + 0.13} & \bestnum{0.30} & \bestnum{0.11} & \bestnum{0.90} & \bestnum{0.68} & \bestnum{0.58} & \bestnum{0.57} \\
\bottomrule\bottomrule
\end{tabular}%
}
\end{table}

We evaluate DPA-GRPO on TaxCalcBench TY24~\citep{bock2025taxcalcbench}, a structured tax-form completion benchmark in which each task instance consists of input records, form rules, and a set of evaluated output fields. Our experiments are designed to answer two questions. First, does the proposed generator--verifier training procedure improve held-out accuracy compared with zero-shot and trained single-agent baselines? Second, do the learned policies change the decision-unit dynamics in the way predicted by our game formulation, rather than merely improving a single aggregate score?

\subsection{Experimental Setup}
\label{sec:exp-setup}

\paragraph{Benchmark and evaluation units.}
TaxCalcBench TY24 provides synthetic U.S. federal tax-form instances with ground-truth Form 1040 outputs produced by a deterministic tax engine. We use a fixed-seed 80/20 train/test split. Although each tax return contains many form fields, our evaluation follows the benchmark protocol and scores 19 selected Form 1040 lines per case. A decision unit corresponds to a structured output component, such as a form field or tax-form line, that can be checked against the ground-truth Form 1040 output. We measure accuracy as the fraction of evaluated decision units whose submitted value exactly matches the ground-truth value.

\paragraph{Models and methods.}
We evaluate Qwen3-4B and Qwen3-8B backbones \citep{yang2025qwen3}. The zero-shot baselines use the generator directly without parameter updates and without the verifier loop. The trained single-agent baselines update only the generator using GRPO-style objectives, including GRPO \citep{shao2024deepseekmath}, DAPO \citep{yu2025dapoopensourcellmreinforcement}, and GSPO \citep{zheng2025groupsequencepolicyoptimization}. These baselines produce a single proposal for each decision unit and are evaluated by proposal correctness.

DPA-GRPO trains a generator and a verifier in the two-player loop described in Section~\ref{sec:game}. For each decision unit, the generator first produces an initial proposal. The verifier either raises an SAC or remains silent. If an SAC is raised, the generator may produce a candidate revision and choose whether to keep the original proposal or revise. We report the final submitted output after this interaction. Thus, for baselines, the reported accuracy is generator-only proposal accuracy; for DPA-GRPO, it is post-interaction submitted-output accuracy.

\paragraph{Training and metrics.}
All trained methods use the same train/test split and comparable LoRA fine-tuning settings. DPA-GRPO constructs paired counterfactual action groups for the verifier, $\{\mathrm{SAC},\mathrm{NS}\}$, and for the generator, $\{\mathrm{KEEP},\mathrm{REVISE}\}$ when an SAC is raised. Alongside test accuracy, we report training-time accuracy on three common training samples, denoted \textsc{Train-A}, \textsc{Train-B}, and \textsc{Train-C}, that appear across trained runs. These columns serve as representative diagnostics of the learned behavior during training and complement the test evaluation.

\begin{figure}[t]
  \centering
  \includegraphics[width=\linewidth]{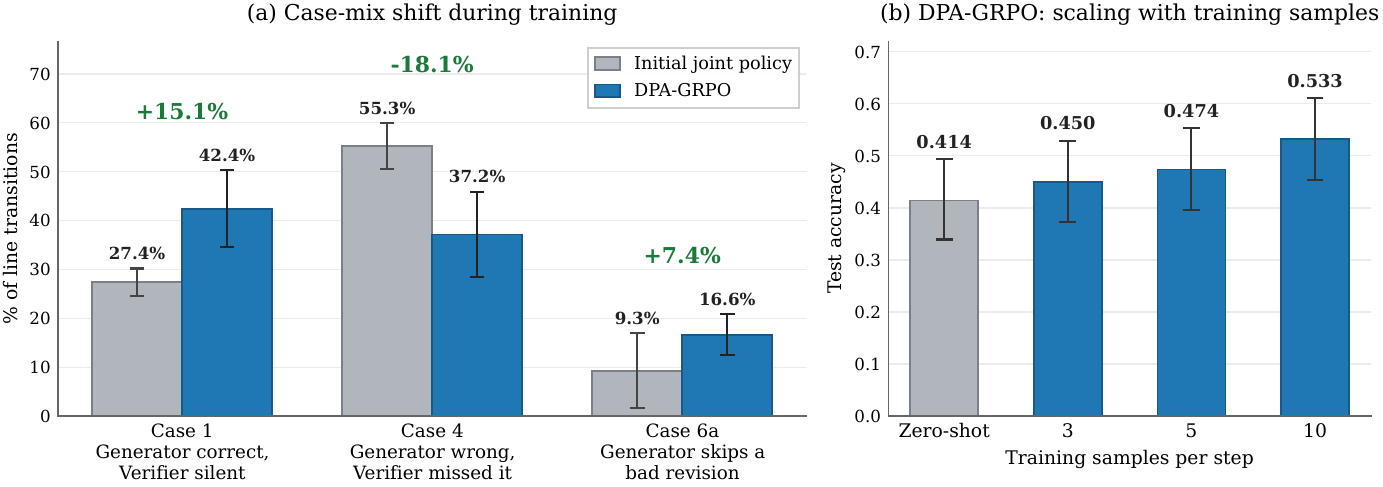}
    \caption{
\textbf{DPA-GRPO improves case-taxonomy dynamics and benefits from more training samples.}
(a) We compare the initial generator-verifier policy with the trained DPA-GRPO policy on three taxonomy categories. DPA-GRPO increases Case~1, corresponding to correct proposals accepted without intervention; reduces Case~4, corresponding to missed errors; and increases Case~6a, where the generator skips a bad revision. Bars report pooled case proportions over a 5-step window, and error bars show $\pm 1$ standard deviation of per-step proportions within that window.
(b) Scaling with training samples. Test accuracy improves with increasing training samples per step, and DPA-GRPO outperforms the zero-shot baseline across all training settings. Error bars show Wilson 95\% confidence intervals for the binomial uncertainty of the evaluated decision-unit outcomes.
}
\label{fig:case-and-scaling}
  \label{fig:case-improvement}
\end{figure}

\subsection{Comparison with Baselines}
\label{sec:exp-baselines}

Table~\ref{tab:main-results} reports the headline comparison. DPA-GRPO achieves the best evaluation accuracy at both backbone sizes. With Qwen3-4B, DPA-GRPO reaches 0.53 test accuracy, improving over zero-shot generation (0.41) and the trained generator-only baselines. With Qwen3-8B, DPA-GRPO reaches 0.57, again outperforming zero-shot generation (0.53) and the trained generator-only baselines. 

The comparison with generator-only RL baselines isolates the contribution of the two-player training structure. All trained baselines update the generator using reinforcement-learning objectives, but they do not train a verifier intervention policy or a generator keep/revise policy. In contrast, DPA-GRPO trains both roles through paired SAC/no-SAC and KEEP/REVISE decisions. The consistent improvement across Qwen3-4B and Qwen3-8B suggests that the generator-verifier interaction provides useful training signals beyond generator-only optimization. Figure~\ref{fig:case-and-scaling}(b) further examines how DPA-GRPO changes with the number of training samples per step. For Qwen3-4B, test accuracy improves as the number of training samples increases, and every trained DPA-GRPO setting outperforms the zero-shot baseline. Error bars report Wilson 95\% confidence intervals for binomial uncertainty over the evaluated decision units in the fixed test split. This suggests that the paired generator--verifier objective benefits from additional task diversity during training.

The case-taxonomy columns provide a test-set diagnostic decomposition of the same results. For generator-only baselines, Case 1 + Case 3 and Case 4 reduce to correct and incorrect generator outputs, respectively, while Case~6a is not applicable because these methods do not use verifier interventions. For DPA-GRPO, the same columns reflect the full generator-verifier interaction, including missed errors and revision-selection behavior. DPA-GRPO reduces the missed-error category: Case~4 drops to 0.40 for Qwen3-4B and 0.30 for Qwen3-8B. It also produces nonzero Case 6a rates, indicating that the generator sometimes rejects candidate revisions that would not improve the output, rather than revising whenever an SAC is raised.

The training-case columns show the same overall pattern. DPA-GRPO obtains the highest reported accuracy across the three common training samples. For Qwen3-8B, DPA-GRPO reaches 0.90, 0.68, and 0.58 on \textsc{Train-A}, \textsc{Train-B}, and \textsc{Train-C}, respectively. Together with the test results, these diagnostics support the main empirical claim: DPA-GRPO improves structured decision accuracy by jointly training the verifier's intervention policy and the generator's response policy, rather than relying solely on generator-only reinforcement learning.

\subsection{Case Taxonomy Analysis}
\label{sec:exp-taxonomy}

Aggregate accuracy alone does not explain how the generator-verifier interaction changes during training. To diagnose the learned behavior, we decompose each transition using the taxonomy from Section~\ref{sec:taxonomy}. Figure~\ref{fig:case-improvement}(a) compares the initial joint policy with the trained DPA-GRPO policy on three diagnostic cases. Bars report pooled case proportions over a 5-step window, and error bars denote $\pm 1$ standard deviation of the per-step proportions within that window.

The first panel shows Case~1, where the generator is correct, and the verifier remains silent. This is the desired terminal behavior. The frequency of Case~1 increases from $27.4\%$ to $42.4\%$, a gain of $+15.1$ percentage points. This indicates that DPA-GRPO does not merely rely on the verifier to repair mistakes after they occur. Instead, training increases the probability that the generator produces correct proposals that the verifier accepts.

The second panel shows Case~4, where the generator is wrong, and the verifier fails to intervene. This is the dominant failure mode of the initial joint policy: an incorrect proposal passes through without correction. DPA-GRPO reduces Case~4 from $55.3\%$ to $37.2\%$, a drop of $-18.1$ percentage points, showing that the verifier becomes less likely to miss incorrect proposals.

The third panel shows Case~6a, where the initial proposal and candidate revision are both incorrect, and the generator chooses KEEP rather than accepting the bad revision. Case~6a increases from $9.3\%$ to $16.6\%$, a gain of $+7.4$ percentage points. This indicates that DPA-GRPO improves the generator's calibration over keep/revise decisions, rather than simply encouraging revision whenever an SAC is raised.
DPA-GRPO improves not only error detection, but also the generator's ability to make calibrated keep/revise decisions.

Taken together, the taxonomy analysis shows that the accuracy gains in Table~\ref{tab:main-results} are supported by meaningful changes in the interaction dynamics. DPA-GRPO increases correct silent acceptance, reduces missed errors, and improves the generator’s calibration over revision decisions. These shifts are aligned with the intended structure of the two-player game: the verifier learns when to challenge, while the generator learns when to accept or reject that challenge.

\section{Conclusion}

We proposed DPA-GRPO, a two-player generator-verifier framework for verifiable assurance reasoning in structured LLM workflows. By representing verifier feedback as safety assurance cases and converting intervention/revision decisions into paired counterfactual actions, DPA-GRPO provides training signals for both agents while preserving interpretability. Our analysis connects the induced case taxonomy to local Nash equilibrium conditions and stochastic-approximation dynamics. Experiments on TaxCalcBench TY24 show that DPA-GRPO improves structured decision accuracy over zero-shot generation and generator-only RL baselines across Qwen3-4B and Qwen3-8B backbones. The decision-unit taxonomy further shows that training increases correct silent acceptance, reduces missed errors, and improves calibrated responses to verifier interventions. These results suggest that structured verifier interventions can serve not only as explanations, but also as actionable training signals for improving both generator and verifier behavior.

{
\small
\nocite{*}
\bibliographystyle{unsrtnat}
\bibliography{ref}
}
\newpage
\appendix 

\section{Related Work}
\label{sec:relatedwork}

\paragraph{RL-based post-training and GRPO.}
RL-based post-training commonly optimizes a KL-regularized objective against a reference policy, often using PPO-style updates or variants \citep{schulman2017proximal,ouyang2022training}. GRPO is a prominent alternative that avoids explicit value-function training by normalizing rewards within groups of samples, which can improve stability and throughput for verifiable reward settings \citep{shao2024deepseekmath}. Our work adopts GRPO as the primary optimization primitive and uses KL regularization to bound policy drift.

\paragraph{Preference optimization and GRPO--DPO connections.}
Direct Preference Optimization (DPO) optimizes policies from preference pairs via an implicit reward view, offering a supervised-like alternative to RL-style updates \citep{rafailov2023direct}. Recent work connects GRPO to contrastive preference optimization and shows that pairwise grouping, or 2-GRPO, can retain key optimization behavior \citep{wu2025takes}. Related reward-matching perspectives further connect group-normalized RL objectives with implicit-reward formulations from preference optimization \citep{wang2025gift}. Our setting naturally induces paired actions for both roles: SAC/no-SAC for the verifier and keep/revise for the generator. This allows us to construct counterfactual dual-action GRPO updates that resemble preference optimization, while still using explicit verifiable rewards rather than human preference labels.

\paragraph{Multi-agent critique, verification, and iterative refinement.}
A broad literature explores self-critique, debate, and multi-agent refinement for improving reliability, often via prompting and heuristic selection \citep{madaan2023self,shinn2023reflexion,du2024improving,khan2024debating,lifshitz2025multi,zhang2024modeling, tang2026agent,jiang2025agentic}. We treat verification as an explicit decision policy (SAC vs.\ no-SAC) trained with verifiable signals, and we train the generator to respond to interventions through revision decisions. 

\paragraph{LLMs as agents in two-player games.}
Recent works have framed LLM reasoning and post-training through two-player or game-theoretic interactions, drawing on classical notions such as Nash equilibrium and stochastic games, \citep{nash1950equilibrium,shapley1953stochastic}. \cite{liu2024largelanguagemodelsagents} interpret LLM training and interaction as a language-based two-player game, where each agent has its own policy, reward, and best-response dynamics. \cite{kempinski2025game} study iterative self-play reasoning in strategic domains and show that naive LLM action selection can be highly exploitable, while best-response-style refinement improves robustness. \cite{cui2026game} formulate information-seeking as a two-player zero-sum extensive-form game and approximate Nash-equilibrium strategies through game-theoretic search. These works support a game-theoretic view of generator--verifier interaction, but our method differs by grounding the game in structured output components, with paired SAC/no-SAC and keep/revise actions trained using verifiable counterfactual rewards.

\section{Limitations and Broader Impact}
\label{sec:limitations}

DPA-GRPO is evaluated on TaxCalcBench TY24, a controlled benchmark with deterministic ground-truth Form 1040 outputs. This setting is useful for studying verifiable structured decisions, but future work should evaluate the method on additional domains with noisier, incomplete, or partially subjective supervision. Our theoretical analysis is local and relies on standard stochastic-approximation and game-theoretic assumptions; it does not imply global convergence or guarantee that finite-step LLM training reaches a Nash equilibrium. Empirically, performance may depend on model capacity, reward design, SAC quality, and the frequency of useful revision opportunities.

The method also incurs additional computational cost relative to generator-only training because it trains both a generator and a verifier and constructs paired counterfactual action groups. While LoRA adapters reduce this cost, scaling to larger models and longer workflows may require more efficient sampling and verifier-scoring strategies. From a broader-impact perspective, SACs can improve auditability by exposing claim-argument-evidence records, but they should not be interpreted as formal guarantees of correctness or safety. In high-stakes settings, outputs should remain subject to expert review and independent validation.
\section{Algorithm}
\label{sec:algorithm}

\begin{algorithm}[H]
\caption{DPA-GRPO training for the generator--verifier game.}
\label{alg:training}
\begin{algorithmic}[1]
\REQUIRE Policies $\pi_{G,\phi},\pi_{V,\theta}$, oracle $S$,
         reference policies $\pi_{G,\mathrm{ref}},\pi_{V,\mathrm{ref}}$,
         KL coefficient $\beta$, step sizes $\eta_k$
\FOR{iteration $k=1,2,\dots$}
    \STATE Sample a batch of task instances and decision units.
    \FOR{each decision context $m_t=(\tau,u_t,o_{<t})$}
        \STATE Sample initial proposal
               $x_t\sim\pi^x_{G,\phi}(\cdot\mid m_t)$ and
               compute $S_x \leftarrow S(x_t;\tau,u_t)$.
        \STATE Set $c_{\mathrm{sac}}\leftarrow 1-S_x$.
        \STATE Sample verifier action
               $y_t\sim\pi_{V,\theta}(\cdot\mid m_t,x_t)$.
        \STATE Construct verifier pair $\{\mathrm{SAC},\mathrm{NS}\}$ and
               compute counterfactual paired rewards
               $R_V(\mathrm{SAC})=c_{\mathrm{sac}},\ R_V(\mathrm{NS})=1-c_{\mathrm{sac}}$.
        \IF{$y_t=\mathrm{SAC}$}
            \STATE Sample candidate revision
                   $z_t\sim\pi^z_{G,\phi}(\cdot\mid m_t,x_t,y_t)$ and
                   compute $S_z \leftarrow S(z_t;\tau,u_t)$.
            \STATE Construct generator pair
                   $\{\mathrm{KEEP},\mathrm{REVISE}\}$ and compute
                   counterfactual paired rewards
                   $R_G(\mathrm{KEEP})=S_x,\ R_G(\mathrm{REVISE})=S_z$.
            \STATE Sample revision action
                   $a_t\sim\pi^a_{G,\phi}(\cdot\mid m_t,x_t,y_t,z_t)$.
            \STATE Set $o_t\leftarrow x_t$ if $a_t=\mathrm{KEEP}$, and
                   $o_t\leftarrow z_t$ if $a_t=\mathrm{REVISE}$.
        \ELSE
            \STATE Set $z_t\leftarrow\varnothing$, $a_t\leftarrow\mathrm{KEEP}$, and
                   $o_t\leftarrow x_t$.
        \ENDIF
        \STATE Store decision-unit transition, paired rewards, and log-probabilities.
    \ENDFOR
    \STATE Update $\theta$ using REINFORCE-style policy gradients on verifier
           SAC/NS paired rewards with $\beta$-KL anchoring to $\pi_{V,\mathrm{ref}}$.
    \STATE Update $\phi$ using REINFORCE-style policy gradients on generator
           proposal and KEEP/REVISE paired rewards with $\beta$-KL anchoring to
           $\pi_{G,\mathrm{ref}}$.
\ENDFOR
\end{algorithmic}
\end{algorithm}

\section{Theory Guarantees}
\label{app:theory}

This section contains theorem-specific technical assumptions and proofs for
Theorems~\ref{thm:main-lne} and~\ref{thm:main-convergence}.

\subsection{Proof of Local Nash Necessary Conditions}
\label{app:lne-proof}

The main text states the local non-degeneracy assumption because it is central to the interpretation
of the case taxonomy. We add here the KL-specific assumption required to transfer reward-improving
directions to the regularized objectives.

\begin{assumption}[KL-dominance for reward-improving deviations]
\label{ass:kl-dominance}
For every local direction identified by Lemma~\ref{lem:main-verifier-br} or
Lemma~\ref{lem:main-generator-br} that strictly improves the corresponding unregularized reward,
the first-order reward gain is not smaller than the first-order increase in the corresponding KL
penalty. Equivalently, the same direction improves the full KL-regularized objective.
\end{assumption}

\begin{proof}[Proof of Theorem~\ref{thm:main-lne}]
Suppose, for contradiction, that $(\phi^*,\theta^*)$ is a local Nash equilibrium but the verifier
assigns nonzero probability to a pointwise reward-suboptimal action on a positive-measure set of
contexts. By Lemma~\ref{lem:main-verifier-br}, there exists a local direction $\Delta\theta$ such that
$D_\theta \widetilde J_v(\theta^*;\phi^*)[\Delta\theta]>0$. By Assumption~\ref{ass:kl-dominance},
the same direction also increases $J_v$, contradicting local Nash optimality of $\theta^*$.

Now suppose the generator assigns nonzero probability to a strictly suboptimal KEEP/REVISE action on a
positive-measure set of SAC contexts. By Lemma~\ref{lem:main-generator-br}, there exists a local direction
$\Delta\phi$ such that $D_\phi \widetilde J_f(\phi^*;\theta^*)[\Delta\phi]>0$. By
Assumption~\ref{ass:kl-dominance}, the same direction also increases $J_f$, contradicting local Nash
optimality of $\phi^*$. Therefore both conclusions must hold.
\end{proof}

\paragraph{Case-taxonomy consequence.}
Theorem~\ref{thm:main-lne} rules out positive equilibrium mass on verifier mistakes and strictly
suboptimal SAC responses. Case~5a can remain admissible because a wrong proposal followed by correct
SAC and successful revision can be reward-consistent for both players. A stronger Case-1-only result
requires an additional assumption, such as an explicit intervention/revision cost or a dominance
condition that producing a correct proposal is strictly preferred to relying on the SAC--revision pathway.

\subsection{Proof of Convergence Statement}
\label{app:convergence-proof}

For the convergence result, we use standard stochastic-approximation assumptions.

\begin{assumption}[Smoothness and compactness]
\label{ass:smoothness}
The parameter domains for $\phi$ and $\theta$ are compact, or the iterates are projected to a compact
set. The objectives $J_f$ and $J_v$ are continuously differentiable and have Lipschitz gradients on this set.
\end{assumption}

\begin{assumption}[Bounded noise]
\label{ass:bounded-noise}
Rewards and GRPO advantages are uniformly bounded. The stochastic gradient noise terms have bounded
second moments and form martingale-difference sequences with respect to the training filtration.
\end{assumption}

\begin{assumption}[Step sizes]
\label{ass:stepsizes}
The learning rates satisfy the Robbins--Monro conditions
\[
\sum_{t=1}^{\infty}\eta_t=\infty,\qquad
\sum_{t=1}^{\infty}\eta_t^2<\infty.
\]
\end{assumption}

\begin{assumption}[Controlled replay and stale-sample bias]
\label{ass:controlled-bias}
Let $b^f_t$ and $b^v_t$ denote bias terms due to replay, stale log-probabilities, and finite group
normalization. We assume
\[
\sum_{t=1}^{\infty}\eta_t\left(\|b^f_t\|+\|b^v_t\|\right)<\infty.
\]
\end{assumption}

\begin{assumption}[Infinite visitation]
\label{ass:visitation}
Each role is updated infinitely often, and each decision branch appearing in the objectives is visited
infinitely often. In particular, the SAC branch has nonzero visitation frequency along the training trajectory.
\end{assumption}

The alternating GRPO iterates can be written as
\[
\phi_{t+1}
=
\phi_t+\eta_t\left(\nabla_\phi J_f(\phi_t;\theta_t)+M^f_{t+1}+b^f_t\right),
\]
\[
\theta_{t+1}
=
\theta_t+\eta_t\left(\nabla_\theta J_v(\theta_t;\phi_{t+1})+M^v_{t+1}+b^v_t\right),
\]
where $M^f_{t+1}$ and $M^v_{t+1}$ are martingale-difference noise terms and $b^f_t,b^v_t$
collect finite-sample, replay, and stale-log-probability bias.

\begin{proof}[Proof sketch of Theorem~\ref{thm:main-convergence}]
By Assumptions~\ref{ass:bounded-noise} and~\ref{ass:stepsizes}, the martingale-difference noise terms
average out asymptotically. Assumption~\ref{ass:controlled-bias} ensures that replay, stale log-probabilities,
and finite group-normalization effects do not accumulate enough to alter the limiting ODE.

The update is block-coordinate: the verifier update uses $\phi_{t+1}$ instead of $\phi_t$. Since
$\phi_{t+1}-\phi_t=O(\eta_t)$ and the gradients are Lipschitz by Assumption~\ref{ass:smoothness}, this
introduces an $O(\eta_t)$ perturbation inside the verifier gradient and hence an $O(\eta_t^2)$ perturbation
in the parameter update, which is summable by Assumption~\ref{ass:stepsizes}. Therefore the interpolated
process tracks the ODE
\[
\dot\phi=\nabla_\phi J_f(\phi;\theta),\qquad
\dot\theta=\nabla_\theta J_v(\theta;\phi).
\]
Standard stochastic-approximation results imply convergence to an internally chain transitive set of this ODE.
Any isolated asymptotically stable limit point of the ODE is stationary, yielding
$\nabla_\phi J_f(\phi;\theta)=0$ and $\nabla_\theta J_v(\theta;\phi)=0$.
\end{proof}

\section{Implementation Details}
\label{app:implementation}

All experiments use the same fixed-seed 80/20 train/test split of TaxCalcBench TY24. We evaluate Qwen3-4B and Qwen3-8B backbones and train all methods using LoRA adapters with AdamW. For DPA-GRPO, the generator and verifier use separate LoRA adapters over the same backbone family.

\paragraph{LoRA configuration.}
We use LoRA rank $r=8$, scaling parameter $\alpha=32$, and dropout $0.05$. Separate generator and verifier adapters are applied to the attention and MLP linear projections:

\paragraph{Optimization.}
We use AdamW with learning rate $10^{-5}$ and weight decay $0.01$. DPA-GRPO uses a REINFORCE-style policy-gradient update with a KL anchor to a frozen reference policy. The KL coefficient is $\beta=0.04$, applied uniformly across role and branch updates.

\paragraph{Sampling and paired-action groups.}
The paired-action group size is $2$. For the verifier, the paired actions are
$
\{\mathrm{SAC},\mathrm{NS}\},
$
and for the generator on SAC branches, the paired actions are
$
\{\mathrm{KEEP},\mathrm{REVISE}\}.
$
We use temperature $0.2$ and top-$p=0.9$ for initial proposal and verifier-action sampling. Candidate revisions are sampled with temperature $0.7$ using best-of-$K$ revision generation with $K=5$ candidates, where candidates are filtered for parse validity and then scored by oracle correctness.

\paragraph{Training schedule.}
We train for up to $30$ steps and evaluate every $5$ steps. 

\paragraph{Hardware.}
Experiments are implemented in Python 3 and run on a Linux machine with an AMD EPYC 7513 32-Core Processor CPU and a single NVIDIA RTX A6000 GPU. We use mixed-precision training with bf16 autocast and gradient checkpointing.

\section{Illustrative line-level interactions}
\label{app:case-interactions}

Each example is one \emph{decision unit}: a single Form~1040 line in one
TaxCalcBench case. \emph{Generator} emits the line JSON; \emph{Verifier}
emits a structured self-assessment with per-line
\texttt{verdict}$\in\{$\texttt{correct}, \dots$\}$; we map that to
\texttt{NO\_SAC} vs.\ \texttt{SAC}. When \texttt{SAC} fires, the Generator
produces a \emph{revision} candidate and the \emph{Approver}
(\texttt{KEEP} / \texttt{REVISE}) chooses which JSON is committed.
Oracle tags $S_x,S_z\in\{0,1\}$ indicate strict match of draft / revision
to gold for that line (training-time shaping).

\paragraph{Case 1.}
\textit{Case} \texttt{single-w2-multiple-1099int-withholding-schedule-b},
\textit{line}~1a (AGI).
\begin{quote}\small\ttfamily\sloppy
\textbf{Generator:} AGI $= 87{,}451$; rationale: wages with no adjustments.\\
\textbf{Verifier:} NO\_SAC; line finding \texttt{correct}; cites W-2 Box~1.\\
\textbf{Revision:} not considered (no SAC).\\
\textbf{Oracle:} $S_x{=}1$ (draft matches gold).
\end{quote}
\textit{Reading:} both roles agree the line is fine; the Verifier does not
intervene. This is the desirable ``Generator correct, Verifier silent''
mass that Table~1 tracks as Case~1.

\paragraph{Case 4.}
\textit{Case} \texttt{single-w2-unemployment-1099g}, \textit{line}~1a (AGI).
\begin{quote}\small\ttfamily\sloppy
\textbf{Generator:} AGI $=131{,}131$; rationale ties the amount to
``prior-year AGI'' from the intake header.\\
\textbf{Verifier:} NO\_SAC; line finding still marked \texttt{correct}
in the JSON (false reassurance).\\
\textbf{Revision:} not considered.\\
\textbf{Oracle:} $S_x{=}0$ (draft does \emph{not} match gold on this line).
\end{quote}
\textit{Reading:} the draft is wrong, but the Verifier never raises SAC, so
no repair loop can run. This is Case~4, the dominant failure mode of the
initial joint policy.
\paragraph{Example 5a.}
\textit{Case} \texttt{single-multiple-w2-excess-social-security-tax-same-ein},
\textit{line}~15 (taxable income).
\begin{quote}\small\ttfamily\sloppy
\textbf{Generator (draft):} amount $146{,}000$ while the text argues
AGI $-$ standard deduction $= 147{,}097$ (numeric inconsistency).\\
\textbf{Verifier:} SAC; line finding \texttt{wrong}; disputes arithmetic.\\
\textbf{Generator (revise):} amount $147{,}097$, aligned with stated AGI and
standard deduction.\\
\textbf{Approver:} REVISE (commit the revised line).\\
\textbf{Oracle:} $S_x{=}0$, $S_z{=}1$; justified SAC.
\end{quote}

\begin{figure}[h]
    \centering
    \includegraphics[width=\linewidth]{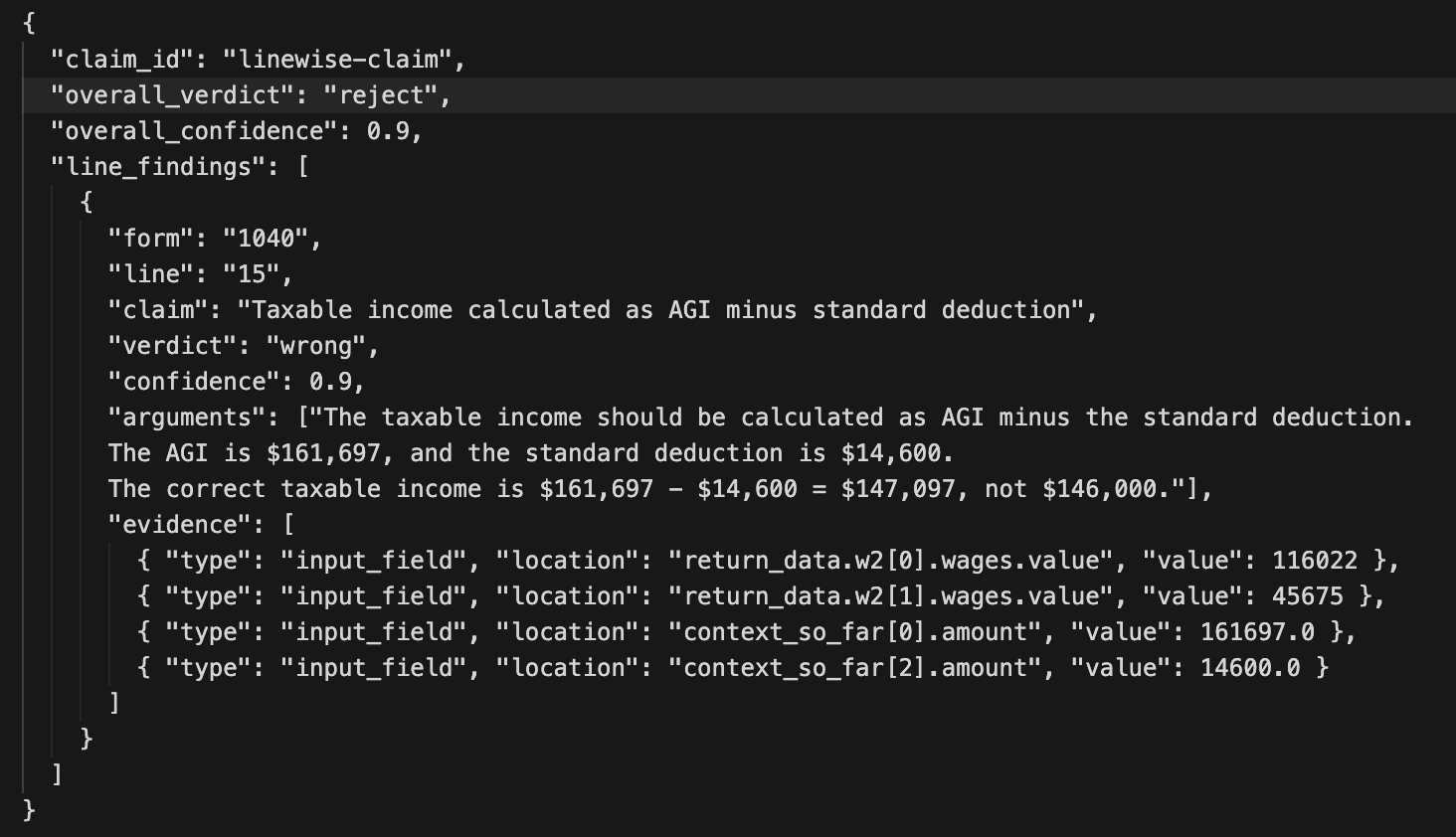}
    \caption*{{\textbf{Case 5a SAC.}
}}
\end{figure}

\textit{Reading:} the Verifier correctly flags a real error; the revised
line strictly improves over the draft ($S_z{>}S_x$ in the usual partial
ordering); the Approver accepts the repair. Taxonomy bucket~5a: wrong
draft, valid critique, revision fixes the line.

\paragraph{Example 6a.}
\textit{Case}
\texttt{single-w2-multiple-1099int-withholding-schedule-b},
\textit{line}~11 (tax liability).
\begin{quote}\small\ttfamily\sloppy
\textbf{Generator (draft):} liability $12{,}222.30$; rationale mixes
bracket steps and contradicts the headline number.\\
\textbf{Verifier:} SAC; line finding \texttt{wrong}; argues bracket /
tax-table mistake.\\
\textbf{Generator (revise):} proposes \emph{different} liability
$11{,}080.34$ (still not gold).\\
\textbf{Approver:} KEEP (retain the original draft, discard revision).\\
\textbf{Oracle:} $S_x{=}S_z{=}0$; justified SAC; approver choice is
counterfactually optimal among $\{$draft, revision$\}$.
\end{quote}

\textit{Reading:} the Verifier is right to be suspicious, but swapping in
the revision would not improve strict correctness; KEEP implements
``skip a bad revision.'' This is Case~6a and is \emph{not} achievable
without a trained Verifier + Approver loop.

\newpage


\newpage

\end{document}